\documentclass[11pt]{article}
\usepackage[numbers, compress]{natbib}
\usepackage[margin=1in]{geometry}

\usepackage[table]{xcolor}  
\definecolor{Gray}{gray}{0.9}
\definecolor{midgreen}{rgb}{0.1,0.5,0.1}
\definecolor{darkgray}{gray}{0.25}
\definecolor{lightblue}{rgb}{0.25,0.25,0.8}
\definecolor{mydarkblue}{rgb}{0,0.08,0.45}
\usepackage[colorlinks,
linkcolor=lightblue,
anchorcolor=blue,
urlcolor=mydarkblue,
citecolor=lightblue]{hyperref}

\usepackage{setspace}

\usepackage{amsmath}
\usepackage{amssymb}
\usepackage{mathtools}
\usepackage{amsthm}
\usepackage{enumitem}

\usepackage{bm}
\usepackage{algorithmic}

\newtheorem{definition}{Definition}
\newtheorem{corollary}{Corollary}
\newtheorem{theorem}{Theorem}
\newtheorem{lemma}{Lemma}

\newcommand{\email}[1]{\href{mailto:#1}{\color{black} \texttt{#1}}}

\title{SubGen: Token Generation in Sublinear Time and Memory}

\author{
Amir Zandieh\footnote{Equal Contributions} \\ Independent Researcher \\ \email{amir.zed512@gmail.com} 
\and
Insu Han$^{*}$ \\ Yale University \\ \email{insu.han@yale.edu}
\and Vahab Mirrokni \\ Google Research \\ \email{mirrokni@google.com} 
\and Amin Karbasi \\ Yale University, Google Research \\ \email{amin.karbasi@yale.edu} 
}

\date{}

\usepackage[linesnumbered,boxruled,vlined]{algorithm2e}
\usepackage[capitalize]{cleverref}
\usepackage{amsthm}
\usepackage{subfigure}
\usepackage{multirow,makecell}
\usepackage{booktabs}

\usepackage{tikz}
\usepackage{pgfplots}

\newcommand{\alg}{\textsc{SubGen}{}}

\newcommand{\norm}[1]{\ensuremath{\left\| #1 \right\|}}

\newcommand{\normnlr}[2]{\|#1\|_{#2}}

\newcommand{\poly}{\text{poly}}

\def\argmin{\mathop{\rm arg~min}}

\def\0{{\bm 0}}

\def\c{{\bm c}}

\def\k{{\bm k}}

\def\q{{\bm q}}

\def\v{{\bm v}}

\def\x{{\bm x}}
\def\y{{\bm y}}
\def\z{{\bm z}}
\def\A{{\bm A}}
\def\B{{\bm B}}

\def\K{{\bm K}}

\def\M{{\bm M}}

\def\V{{\bm V}}

\def\Ccal{\mathcal{C}}
\def\Dcal{\mathcal{D}}

\def\Mcal{\mathcal{M}}

\def\Scal{\mathcal{S}}

\def\RR{\mathbb{R}}

\def\Nbb{\mathbb N}

\def\Attn{\mathrm{Attn}}
\def\softmax{\mathtt{softmax}}

\begin{document}

\maketitle
\vspace{-0.4in}
\begin{abstract}
Despite the significant success of large language models (LLMs), their extensive memory requirements pose challenges for deploying them in long-context token generation. 
The substantial memory footprint of LLM decoders arises from the necessity to store all previous tokens in the attention module, a requirement imposed by key-value (KV) caching.
In this work, our focus is on developing an efficient compression technique for the KV cache.
Empirical evidence indicates a significant clustering tendency within key embeddings in the attention module. 
Building on this key insight, we have devised a novel caching method with sublinear complexity, employing online clustering on key tokens and online $\ell_2$ sampling on values. The result is a provably accurate and efficient attention decoding algorithm, termed \alg{}. 
Not only does this algorithm ensure a sublinear memory footprint and sublinear time complexity, but we also establish a tight error bound for our approach.
Empirical evaluations on long-context question-answering tasks demonstrate that \alg{} significantly outperforms existing and state-of-the-art KV cache compression methods in terms of performance and efficiency.
\end{abstract}

\section{Introduction} \label{sec-intro}

Large Language Models (LLMs) \cite{achiam2023gpt, touvron2023llama}
play a crucial role in various natural language processing applications, including dialog systems \cite{taori2023stanford, chiang2023vicuna}, coding assistance \cite{chen2021evaluating, roziere2023code}, and image/video generations from text~\cite{radford2021learning,ho2022imagen}.
All of these models rely on the transformer architecture, with the attention mechanism serving as the key component.

To fully harness the capabilities of LLMs, they must demonstrate both efficiency and accuracy in generating long sequences. 
In practical applications, deploying LLMs to generate tokens in an autoregressive manner involves a sequential decoding process, where attention is dynamically applied to each newly generated token. This process effectively constructs the output sequence in a streaming manner, one token at a time.
Therefore,  as the sequence grows, the model has to produce contextually relevant and coherent content.


A common method for autoregressive attention decoding involves the use of key-value (KV) caching, where key and value pairs from \textit{all} preceding tokens are cached and reused to prevent redundant computations. However, this approach faces memory constraints, particularly when handling long sequences. In particular, the memory requirements and runtime for generating each new token increase linearly with context size, posing a significant challenge for efficient processing of extensive sequences. This linear scaling directly impedes practical applicability in real-world scenarios, such as chat systems, where large contexts are often encountered.



In this work, we delve into the primary computational and memory bottleneck of token generation. 
We propose \alg, a novel approach designed to significantly reduce the memory and runtime complexity of token generation, moving from conventional linear growth to sublinear scale.
To summarize, our goal is to answer the following question: 
\begin{center}
   \emph{Can we approximate the attention output in decoding phase \\in sublinear space/time complexity in context length?} 
\end{center}


\subsection{Related Work} \label{sec-related-work}

Recent studies have underscored the need for efficient token generation, particularly with the rise of long-range context datasets. Several recent works have developed efficient strategies for compressing the KV cache.
\citet{zhang2023h} proposed a greedy-type eviction algorithm that dynamically keeps at most $k \ll n$ token embeddings based on the accumulated attention scores where they refer to the Heavy Hitter Oracle (H2O).
\citet{liu2023scissorhands} empirically observed that 
tokens with initially high attention scores tend to stay high during the future generation process. Motivated by this observation, the authors proposed a strategy that only keeps the most recent and pivotal tokens whose attention scores are higher than a threshold.
\citet{ge2023model} proposed an adaptive method of KV cache compression which identifies the intrinsic structures of attention heads and uses them to determine the optimal compression policy. 
\citet{xiao2023efficient} observed that a simple eviction mechanism that keeps only the first few and last few tokens does not degrade much the decoding quality. They additionally proposed a fine-tuning method to solve performance degradation from their method.
\citet{liu2023deja} developed an algorithm that reduces the generation latency by exploiting contextual sparsity.
In addition to algorithmic acceleration, there has also been a line of work optimizing hardware resource configurations~\cite{sheng2023flexgen,hong2023flashdecoding}.
However, to the best of our knowledge, none of these works have achieved an efficient method for KV cache with fully sublinear-time memory space. 

On the lower bound side, achieving subquadratic amortized runtime for producing output embeddings for $n$ tokens in the worst-case instances is likely impossible without making assumptions about the input tokens~\cite{alman2023fast, sarlos2023hardness}. 
Therefore, to achieve fast runtime, it is necessary to rely on certain assumptions about the input tokens.

\subsection{Streaming Attention Problem}
Deployment of LLMs involves performing attention decoding in a streaming fashion. 
More precisely, the stream of tokens is represented as a sequence of vector triplets $(\q_1, \k_1, \v_1), (\q_2, \k_2, \v_2), \ldots (\q_n, \k_n, \v_n)$, where $\q_i, \k_i, \v_i \in \RR^d$ are queries, keys, and values of the attention mechanism and $n$ is the total number of tokens in the stream so far either in prompt or generation. \footnote{We denote vectors with lowercase boldface letters, e.g., $\v$, matrices with uppercase boldface letters, e.g., $\M$, and sets with calligraphy uppercase letters, e.g., $\Scal$. The operator norm of a matrix is denoted as $\norm{\cdot}_{{op}}$.
}
The objective of streaming attention decoding is to compute the following:
\begin{equation}\label{eq:streaming_attention_exact}
    \Attn(\q_n, \K_n, \V_n) = \softmax(\K_n \cdot \q_n)^\top \cdot \V_n,
\end{equation}
where $\K_n, \V_n \in \RR^{n \times d}$ are matrices defined by stacking the keys and values in their respective rows:
\begin{equation}\label{eq:KV_matrices}
\K_n := \begin{bmatrix}
\k_1^\top \\
\k_2^\top \\
\vdots \\
\k_n^\top
\end{bmatrix}, ~~~ 
\V_n := \begin{bmatrix}
\v_1^\top \\
\v_2^\top \\
\vdots \\
\v_n^\top
\end{bmatrix}.
\end{equation}
The output $\Attn(\q_n, \K_n, \V_n)$ is then used for predicting the next token and its token embedding is applied to a transformer model and a new stream pair $(\q_{n+1}, \k_{n+1}, \v_{n+1})$ is generated.
However, storing these values and keys requires $O(n d)$ memory, posing a significant space complexity challenge for long-context models with large $n$.

\subsection{Overview of Contributions}
We propose \alg{}, an efficient method that accurately approximates the attention decoder's output in \cref{eq:streaming_attention_exact} while retaining only a small (sublinear) subset of keys and values in the cache. 
In particular, \alg{} computes an estimator $\z_n$ for $\Attn(\q_n, \K_n, \V_n)$ in sublinear time and memory such that the error is bounded as follows:
\begin{align}
&\norm{\z_n - \Attn(\q_n, \K_n, \V_n)}_2 \le \varepsilon \norm{\softmax(\K_n \cdot \q_n)}_2 \norm{\V_n}_{op}. \label{eq:error_bound_intro}
\end{align}
This error bound is in line with the spectral errors studied in previous works \cite{zandieh2023kdeformer, han2023hyperattention}.

We begin by observing that $\Attn(\q_n, \K_n, \V_n)$ in \cref{eq:streaming_attention_exact} is the product of the softmax vector $\softmax(\K_n \cdot \q_n)$ and value matrix $\V_n$.
This matrix-vector product can be approximated by sub-sampling only $O(\varepsilon^{-2} d \log n)$ key-value pairs according to the vector and matrix according to the squared norms of value tokens. 
This can be implemented in a streaming setting using some variants of reservoir sampling.

The other major computational challenge is computing the partition function in the denominator of the softmax function, i.e., $\sum_{i\in[n]} \exp(\langle \k_i, \q_n \rangle)$. 
To solve this, we construct a data structure that can be stored in sublinear memory and efficiently approximate $\sum_{i\in[n]} \exp(\langle \k_i, \q_n \rangle)$ up to $1 \pm \varepsilon$ factor for any query $\q_n$.
Our method assumes that the key tokens can be covered by a sublinear number of bounded diameter clusters.
This assumption is indeed weaker than the one made in \cite{han2023hyperattention}, which in the decoding setting translates to having key tokens belong to only one cluster with a bounded diameter, while our approach allows for any sublinear number of clusters.
So, if the keys are composed of bounded diameter clusters then we only need a small number of uniformly sampled keys from each cluster to approximate the softmax normalizer efficiently and accurately.
The central task is to find these clusters in a streaming setting, and we achieve this using an algorithm that is inspired by the streaming k-center algorithm of \cite{charikar1997incremental}.

In \cref{main_theorem} and \cref{main_cor} we demonstrate that if the keys can be clustered into some sublinear number $m = n^{1 - \Omega(1)}$ of clusters with some bounded diameters, then \alg{} operates with sublinear $O\left( \varepsilon^{-2} m d \right) = O\left( \varepsilon^{-2} d n^{1-\Omega(1)} \right)$ memory and runtime and its output satisfies the approximation guarantee in \cref{eq:error_bound_intro}.

In \cref{sec-exp}, we empirically compare \alg{} to other KV cache compression methods including the attention-score-based algorithm of \cite{zhang2023h} and the deterministic eviction policy from \cite{xiao2023efficient}. 
Our results confirm that \alg{} outperforms these methods, particularly in question-answering tasks with various sequence lengths.

\section{Sublinear Time and Memory Algorithm} \label{alg-sec}

Our goal is to approximate the attention output in \cref{eq:streaming_attention_exact} with a space complexity that is sublinear in context length $n$. To achieve this objective, we aim to design the following data structure (DS) for efficiently approximating the streaming attention mechanism:

\subsection{Streaming Attention Data Structure}
For every positive integer $n$ and every stream of token triplets $(\q_1, \k_1, \v_1), (\q_2, \k_2, \v_2), \ldots (\q_n, \k_n, \v_n)$ where $\q_i, \k_i, \v_i \in \RR^d$, we aim to construct an efficient DS
with the following properties:
\begin{itemize}
    \item The required memory space is sublinear in $n$, i.e., $o(n)$.
    \item Upon the arrival of a new triplet $(\q_{n+1}, \k_{n+1}, \v_{n+1})$ in the stream, the time complexity to update is sublinear in $n$, i.e., $o(n)$.
    \item Given such data structure, there exists an algorithm that outputs an estimator $\z_n \in \RR^d$ in sublinear time $o(n)$ such that:
    \begin{align}
    &\norm{\z_n - \softmax(\K_n \cdot \q_n)^\top \cdot \V_n}_2 \le \varepsilon \norm{\softmax(\K_n \cdot \q_n)}_2 \norm{\V_n}_{op}.
    \end{align}
\end{itemize}

In the rest of this section, our focus is on developing an algorithm to satisfy the above properties.
Note that the attention output in \cref{eq:streaming_attention_exact}, using the definition of softmax, is equivalent to the following expression:
\[
\Attn(\q_n, \K_n, \V_n) = \frac{\exp(\K_n \cdot \q_n)^\top \cdot \V_n}{\sum_{i\in[n]} \exp(\langle \k_i, \q_n \rangle)}.
\]
Thus, to compute the attention output we need to calculate: 
\begin{enumerate}
    \item The matrix-vector product between $\V_n$ and $\exp(\K_n\cdot\q_n)$. 
    \item The partition function $\sum_{i\in[n]} \exp(\langle \k_i, \q_n \rangle)$.
\end{enumerate}
Thus, our DS needs to efficiently approximate each of these two operations. 
The matrix-vector product $\exp(\K_n \cdot \q_n)^\top \cdot \V_n$ can be approximated efficiently using standard sampling-based techniques. 
Specifically, we make use of the row norm sampling approach~\cite{drineas2001fast, cohen2016optimal}. When multiplying two matrices $\A\in\RR^{m\times n}$ and $\B \in \RR^{n \times p}$, we randomly sample an i.i.d. index $i \in [n]$ with probability proportional to the $\ell_2$ norm of the $i$-th row in $\B$. Then, we estimate $\A\cdot\B$ by the average of the product between $i$-th column in $\A$ and $i$-th row in $\B$. With this approximation, we need only $O(\varepsilon^{-2} d \log n)$ samples to guarantee an $\varepsilon$ multiplicative error in spectral norm for $\exp(\K_n \cdot \q_n)^\top \cdot \V_n$.
Luckily, it can be implemented in a streaming setting through a variant of reservoir sampling~\cite{vitter1985random}.

The more challenging task is the sublinear-time approximation of the partition function $\sum_{i\in[n]} \exp(\langle \k_i, \q_n \rangle)$.
We construct a DS for computing this under the assumption that the keys in the token stream are organized into a sublinear ($o(n)$) number of clusters. 
To be more precise, we introduce the following notion of clusterability: 
\begin{definition}[Clusterability]\label{def_clusterable}
    For a positive integer $m$ and a real-valued $\delta > 0$, a dataset of points $\x_1, \x_2, \ldots \x_n \in \RR^d$ is considered {\it $(m, \delta)$-clusterable} if there exists a size-$m$ partition $\Ccal_1, \Ccal_2, \ldots \Ccal_m \subseteq \{ \x_i \}_{i=1}^n$ of the dataset satisfying the following conditions:
    \begin{itemize}
        \item $\Ccal_i \cap \Ccal_j = \emptyset$ for every $i \neq j$ and $\bigcup_{j=1}^m \Ccal_j = \{ \x_i \}_{i=1}^n$.
        \item for every $j \in [m]$ and every distinct pair $\y, \z \in \Ccal_j$, $\norm{\y - \z}_2 \le \delta$.
    \end{itemize}
\end{definition}

We demonstrate that under the assumption that the stream of keys $\k_1, \k_2, \ldots \k_n$ is $(m,\delta)$-clusterable as defined in \cref{def_clusterable}, with the number of clusters scaling sublinearly in stream length ($m = o(n)$), it is possible to construct a DS with sublinear memory space. 
The procedure for this DS is presented in \cref{alg_stresm_attn_ds} which we refer to as \alg.

To verify this in the practical settings, we plot key embeddings from open-source LLMs in \cref{exp-key-embedding} and observe that they are indeed well clusterable on their embedding space. 
This motivates us to utilize an efficient stream clustering algorithm on key embeddings.
In the remainder of this section, we provide a detailed explanation for the execution of the algorithm while simultaneously analyzing it through a series of lemmas.

\begin{algorithm}[!h]
\caption{\alg: Sublinear Streaming Attention} \label{alg_stresm_attn_ds}
\begin{algorithmic}[1] 
    \STATE {\bf inputs:} stream of tokens $(\q_n, \k_n, \v_n)$ for $n \in \Nbb$, parameter $\delta > 0 $, positive integers $s, t$
    \STATE initialize $\mu \gets 0$, $\Dcal \gets \emptyset$, $\Mcal \gets \begin{bmatrix} {\tt null}, \stackrel{\times s}{\cdots\cdots} \end{bmatrix}$
    \REPEAT
    \STATE $\Dcal \gets \textsc{UpdateSoftmaxNormalizer}(\Dcal, \delta, t, \k_n)$
    \STATE $\Mcal \gets \textsc{UpdateMatrixProduct}(\Mcal, s, \mu, \k_n, \v_n)$
    \STATE $\mu \gets \mu + \norm{\v_n}_2^2$ \label{line_mu}
    \STATE $\z_n \gets \textsc{QueryStreamAttn}(\Dcal, \Mcal, s, t, \mu, \q_n)$
    \STATE $n \gets n+1$
    \STATE {\bf output} $\z_n$
    \UNTIL{Token stream ends}
    \vspace{0.02in}
    {\vspace{1mm} \hrule \vspace{1mm}}
    \vspace{0.02in}
    \hspace*{-0.5cm}{\bf Procedure}{ \sc UpdateSoftmaxNormalizer ($\Dcal, \delta, t, \k$)}
    \STATE Suppose input set $\Dcal = \{ (\x_i, \Scal_i, n_i) : i \in [m] \}$
    \STATE $i^* \gets \argmin_{i \in [m]} \norm{\x_i - \k}_2$
    \IF{$\norm{\k - \x_{i^*}}_2 \le \delta$}
    \STATE $n_{i^*} \gets n_{i^*}+1$
    \STATE Suppose $\Scal_{i^*}$ is a list of $t$ vectors in $\RR^d$
    \FOR{$j \in [t]$}
    \STATE Flip a coin and with probability $p=\frac{1}{n_{i^*}}$, update the $j^{th}$ entry of $\Scal_{i^*}$ as $\Scal_{i^*}(j) \gets \k$
    \ENDFOR
    \ELSE
    \STATE $\Scal' \gets \begin{bmatrix} \k, \stackrel{\times t}{\cdots\cdots} \end{bmatrix}$ (contains $t$ copies of $\k$)
    \STATE $\Dcal = \Dcal \cup \{ (\k, \Scal', 1) \}$
    \ENDIF
    \STATE {\bf return} $\Dcal$

    \vspace{0.02in}
    {\vspace{1mm} \hrule \vspace{1mm}}
    \vspace{0.02in}

    \hspace*{-0.5cm}{\bf Procedure}{ \sc UpdateMatrixProduct ($\Mcal, s, \mu, \k, \v$)}
    \STATE Suppose $\Mcal$ is a list of $s$ tuples of vectors in $\RR^d$
    \FOR{$i \in [s]$}
    \STATE Flip a coin and with probability $p=\frac{\norm{\v}_2^2}{\mu+\norm{\v}_2^2}$, update the $i^{th}$ entry of $\Mcal$ as $\Mcal(i) \gets (\k, \v)$
    \ENDFOR
    \STATE {\bf return} $\Mcal$

    \vspace{0.02in}
    {\vspace{1mm} \hrule \vspace{1mm}}
    \vspace{0.02in}

    \hspace*{-0.5cm}{\bf Procedure}{ \sc QueryStreamAttn ($\Dcal, \Mcal, s, t, \mu, \q$)}       
    \STATE $\z \gets \sum_{(\k, \v) \in \Mcal} \frac{\mu}{s \cdot \norm{\v}_2^2} \cdot \exp(\langle \q, \k \rangle) \cdot \v$
    \STATE $\tau \gets \sum_{(\x, \Scal, n') \in \Dcal} \frac{n'}{t} \cdot \sum_{\k \in \Scal} \exp(\langle \q, \k \rangle)$
    \STATE {\bf return} $\z / \tau$
\end{algorithmic}
\end{algorithm}

\subsection{Matrix Product Data Structure}
Here, we focus on the \textsc{UpdateMatrixProduct} primitive and establish its correctness by introducing invariants that are maintained throughout the stream processing.
This primitive maintains and updates a list of $s$ elements denoted by $\Mcal$ in \alg{}~(\cref{alg_stresm_attn_ds}). 
Initially, this list is filled with ${\tt null}$ values. After processing the first token tuple $(\q_1, \k_1, \v_1)$, this list is populated with $s$ copies of the first key and value $(\k_1, \v_1)$. 
The procedure \textsc{UpdateMatrixProduct} performs a variant of reservoir sampling upon observing any new token in the stream. 
At any iteration $n$ of the stream, $\Mcal$ is ensured to contain $s$ i.i.d. samples chosen at random from $(\k_1, \v_1), (\k_2, \v_2), \ldots (\k_n, \v_n)$ with probabilities proportional to $\norm{\k_i}_2^2$. 
More precisely, the following invariants hold:
\begin{lemma}[Correctness of \textsc{UpdateMatrixProduct}]\label{lem_update_matrix_prod}
    For any positive integer $s$, at any iteration $n$ of the stream in \ref{alg_stresm_attn_ds} the following properties are maintained:
    \begin{itemize}
        \item $\mu = \sum_{i \in [n]} \norm{\v_i}_2^2$.
        \item $\Mcal$ is a list of $s$ i.i.d. samples from $\{(\k_1, \v_1), (\k_2, \v_2), \ldots (\k_n, \v_n)\}$ where the probability distribution for each element $j \in [s]$ is $\Pr[\Mcal(j) = (\k_i, \v_i)] = \frac{\norm{\v_i}_2^2}{\sum_{l \in [n]} \norm{\v_l}_2^2}$ for $i \in [n]$.
    \end{itemize}
\end{lemma}
\begin{proof}
    The first property is trivial because $\mu$ is initialized at zero and is updated in line 6 of the algorithm by adding the squared norms of $\v_i$'s.
    The proof of the second invariance is by induction.
    The base of induction holds for $n=1$ because after processing the first token by procedure \textsc{UpdateMatrixProduct} we have $\Pr[\Mcal(j) = (\k_1, \v_1)] = \frac{\norm{\v_1}_2^2}{\norm{\v_1}_2^2}=1$ for $j \in [s]$.

    Now suppose that the inductive hypothesis holds for $n$ and we prove it must also hold for $n+1$. For any $j \in [s]$ in line 24 of \cref{alg_stresm_attn_ds} with probability $p=\frac{\norm{\v_{n+1}}_2^2}{\mu + \norm{\v_{n+1}}_2^2}$, $\Mcal(j)$ gets updated to $(\k_{n+1}, \v_{n+1})$. Since we showed that $\mu = \sum_{i\in[n]}\norm{\v_i}_2^2$ we have:
    \[
    \Pr[\Mcal(j) = (\k_{n+1}, \v_{n+1})] = \frac{\norm{\v_{n+1}}_2^2}{\sum_{l \in [n+1]} \norm{\v_l}_2^2}.
    \]
    Moreover with probability $1-p = \frac{\mu}{\mu+\norm{\v_{n+1}}_2^2}$, $\Mcal(j)$ keeps its previous value.
    Using the inductive hypothesis we have that for every $i \in [n]$:
    \begin{align*}
        \Pr[\Mcal(j) = (\k_{i}, \v_{i})] &= \frac{\norm{\v_{i}}_2^2}{\sum_{l \in [n]} \norm{\v_l}_2^2} \cdot \frac{\sum_{l \in [n]} \norm{\v_l}_2^2}{\sum_{l \in [n+1]} \norm{\v_l}_2^2} = \frac{\norm{\v_{i}}_2^2}{\sum_{l \in [n+1]} \norm{\v_l}_2^2}.
    \end{align*}
    This completes the proof.
\end{proof}

\subsection{Softmax Normalizer (Partition Function) DS}
Here we delve into a detailed discussion of the \textsc{UpdateSoftmaxNormalizer} primitive. 
This primitive constructs and maintains a DS denoted by $\Dcal$, enabling accurate approximation of the partition function in the softmax denominator for any query.
A crucial requirement for the efficiency of this primitive is that the key tokens must be $(m, \delta)$-clusterable, as per \cref{def_clusterable}.
Our algorithm locates and stores a subsampled representation of each cluster in $\Dcal$ in a small memory. 
Particularly, to achieve sublinear memory complexity, instead of keeping all keys in each cluster which would require $O(n)$ memory space, we maintain only a random subset of $t$ samples from each cluster.

Initially, $\Dcal$ is an empty set. 
As new tokens in the stream are processed, new clusters get added to this set.
Each cluster is characterized by a representative point, which is the first key assigned to that cluster by our algorithm.
Throughout stream processing, we compute the distance between the new key token and each existing cluster. 
Here the distance to an existing cluster is defined as the distance to the aforementioned representative of the cluster.
If there is a cluster whose distance is less than $\delta$, then the token is assigned to the nearest cluster, and we update our random samples of keys from this cluster using reservoir sampling.
If the distance from all existing clusters is more than $\delta$, we introduce a new cluster in $\Dcal$, and the new key becomes the representative of this new cluster.
At any point in the stream, this algorithm identifies at most $m$ clusters if the keys so far are $(m, \delta)$-clusterable. If $m$ grows sublinearly in the stream length $n$, the memory and update time of our algorithm will be sublinear as well.
Formally, we prove that the following invariant holds:

\begin{lemma}[Correctness of \textsc{UpdateSoftmaxNormalizer}]\label{lem_update_softmax_norm}
    For any $\delta > 0$, any positive integer $t$, at any iteration $n$ of the stream in \cref{alg_stresm_attn_ds} the following properties are maintained.  $\Dcal$ is a set of $m$ items of the form $\Dcal = \left\{ (\x_i, \Scal_i, n_i) : i \in [m] \right\}$, where there exists a partition of keys into $m$ disjoint subsets $\Ccal_1, \Ccal_2, \ldots \Ccal_m \subseteq \{ \k_i \}_{i=1}^n$ satisfying $\bigcup_{j=1}^m \Ccal_j = \{ \k_i \}_{i=1}^n$ and $\Ccal_i \cap \Ccal_j = \emptyset$ for every $i \neq j$, such that for every $i \in [m]$:
    \setlist[enumerate]{leftmargin=6.5mm}
    \begin{enumerate}
        \item $\x_i \in \Ccal_i$,
        \item $n_i = |\Ccal_i|$,
        \item $\norm{\x_i - \k'}_2 \le \delta$ for every $\k' \in \Ccal_i$,
        \item  $\norm{\x_i - \x_j}_2 > \delta$ for every $i\neq j$,
        \item  $\Scal_i$ is a set of $t$ i.i.d. uniform samples from the set $\Ccal_i$.
    \end{enumerate}
\end{lemma}

\begin{proof}
    The proof is by induction on the stream length $n$. 
    The base of induction trivially holds for $n=0$, where $\Dcal$ is an empty set.
    To prove the inductive step suppose that the inductive hypothesis holds for some $n$. 
    Specifically, suppose that $\Dcal$ is a set of $m$ items of the form $\Dcal = \left\{ (\x_i, \Scal_i, n_i) : i \in [m] \right\}$ and there exists a partition of keys into $m$ disjoint subsets $\Ccal_1, \Ccal_2, \ldots \Ccal_m \subseteq \{ \k_i \}_{i=1}^n$ as per in the lemma statement, such that for every $i \in [m]$: {\bf (1)} $\x_i \in \Ccal_i$, {\bf (2)} $n_i = |\Ccal_i|$, {\bf (3)} $\norm{\x_i - \k'}_2 \le \delta$ for every $\k' \in \Ccal_i$, {\bf (4)} $\norm{\x_i - \x_j}_2 > \delta$ for every $i\neq j$, and {\bf (5)} $\Scal_i$ is a set of $t$ i.i.d. uniform samples from the set $\Ccal_i$.
    Given this assumption, we prove that the inductive step also holds for after processing the $(n+1)$-th key in the stream $\k_{n+1}$.

    In the next iteration, specifically in line 12 of \textsc{UpdateSoftmaxNormalizer}, the algorithm finds the index $i^* \in [m]$ such that $\norm{\x_{i^*} - \k_{n+1}}_2$ is minimized. 
    Two cases arise:
    \paragraph{Case 1: $\norm{\x_{i^*} - \k_{n+1}}_2 \le \delta$.}
    In this case, the algorithm increments $n_{i^*} \gets n_{i^*}+1$ in line 14. Consider the new partitioning of the keys defined as $\Ccal_{i}' = \Ccal_{i}$ for $i \neq i^*$ and $\Ccal_{i^*}' = \Ccal_{i^*} \cup \{ \k_{n+1} \}$. 
    It follows from the inductive hypothesis that for every $i \in [m]$: {\bf (1)} $\x_i \in \Ccal_i'$, {\bf (2)} $n_i = |\Ccal_i'|$, {\bf (3)} $\norm{\x_i - \k'}_2 \le \delta$ for every $\k' \in \Ccal_i'$, and {\bf (4)} $\norm{\x_i - \x_j}_2 > \delta$ for every $i\neq j$ hold after the $n+1$-th iteration.
    Furthermore, since the algorithm does not alter the lists $\Scal_i$ for $i \neq i^*$, we have that {\bf (5)} $\Scal_i$ is a set of $t$ i.i.d. uniform samples from the set $\Ccal_i'$ for any $i \neq i^*$. 
    On the other hand, the algorithm in line 17 performs reservoir sampling on the set $\Scal_{i^*}$ with new element $\k_{n+1}$ which implies that $\Scal_{i^*}$ is a set of $t$ i.i.d. uniform samples from the set $\Ccal_{i^*}'$.
    This completes the inductive step in the first case.
    
    \paragraph{Case 2: $\norm{\x_{i^*} - \k_{n+1}}_2 > \delta$.}
    In this case, the algorithm adds a new element to $\Dcal$, thus, the updated set is $\Dcal' = \{ (\x_i, \Scal_i, n_i): i \in [m+1] \}$ with $\x_{m+1} = \k_{n+1}$ and $n_{m+1} = 1$.
    If we consider the new partitioning of keys to be $\Ccal_1, \Ccal_2, \ldots \Ccal_m, \Ccal_{m+1}$, where $\Ccal_{m+1} = \{ \k_{n+1} \}$, we can use the inductive hypothesis to deduce that for any $i \in [m+1]$: {\bf (1)} $\x_i \in \Ccal_i$, {\bf (2)} $n_i = |\Ccal_i|$, {\bf (3)} $\norm{\x_i - \k'}_2 \le \delta$ for every $\k' \in \Ccal_i$, and {\bf (4)} $\norm{\x_i - \x_j}_2 > \delta$ for every $i \neq j$ hold after the $n+1$-th iteration of the stream.
    Furthermore, $\Scal_{m+1}$ is defined to be a list of $t$ copies of $\k_{n+1}$, thus, {\bf (5)} $\Scal_i$ is a set of $t$ i.i.d. uniform samples from the set $\Ccal_i$ for any $i \in [m+1]$.
This completes the inductive step in this case and also concludes the proof of the lemma.
\end{proof}

\subsection{Streaming Attention: Main Theorem}
Now we are ready to analyze the end-to-end performance of \alg{} and prove the main theorem. 
We show that, given the data structures created throughout the stream and analyzed in \cref{lem_update_matrix_prod} and \cref{lem_update_softmax_norm}, the primitive \textsc{QueryStreamAttn} can efficiently output an accurate approximation to the streaming attention, satisfying \cref{eq:streaming_attention_exact}.

Our analysis unfolds in two steps.
First, we establish that the data structures created by \textsc{UpdateSoftmaxNormalizer} and \textsc{UpdateMatrixProduct} can be stored in small memory and updated very quickly if the sequence of keys is clusterable into a sublinear number of clusters.
Then we show that the \textsc{QueryStreamAttn} can use these data structures to produce an accurate attention output for any given query.
Our main result is as follows:
\begin{theorem}[Efficiency and Correctness of \cref{alg_stresm_attn_ds}]\label{main_theorem}
   For any $\delta, r, \varepsilon > 0$, any positive integers $n, d$, and any sequence of tokens $(\q_1, \k_1, \v_1), (\q_2, \k_2, \v_2), \ldots (\q_n, \k_n, \v_n)$ where $\q_i, \k_i, \v_i \in \RR^d$, suppose that the followings hold 
   \begin{itemize}
       \item $t = \Omega \left( \varepsilon^{-2} \cdot e^{2\delta \cdot r} \log n \right)$,
       \item $s = \Omega(\varepsilon^{-2} \cdot d)$,
       \item $\norm{\q_n}_2 \le r$.
   \end{itemize}
    Then, \alg{}~(\cref{alg_stresm_attn_ds}) at $n$-th step of the stream processing outputs a vector $\z_n \in \RR^d$ that satisfies \cref{eq:streaming_attention_exact} with probability at least $0.99$. 
   Furthermore, if the keys $\k_1, \k_2, \ldots \k_n$ are $(m, \delta)$-clusterable as per \cref{def_clusterable}, then both the total memory of the algorithm and its runtime during the $n$-th iteration is bounded by $O(d \cdot (mt+s))$.
\end{theorem}
\begin{proof}
We start the correctness proof by observing that all preconditions of \cref{lem_update_softmax_norm} are satisfied, allowing us to invoke this lemma.
Let the partition of keys into disjoint subsets be denoted by $\Ccal_1, \Ccal_2, \ldots \Ccal_{m'} \subseteq \{ \k_i \}_{i=1}^n$ satisfying $\bigcup_{j=1}^{m'} \Ccal_j = \{ \k_i \}_{i=1}^n$ and $\Ccal_i \cap \Ccal_j = \emptyset$ for every $i \neq j$ as per \cref{lem_update_softmax_norm} for some positive integer $m'$.
Rewriting the partition function in the attention denominator gives:
\[
\sum_{j\in[n]} \exp(\langle \k_j, \q_n \rangle) = \sum_{i \in [m']} \sum_{\k' \in \Ccal_i} \exp(\langle \k', \q_n \rangle).
\]
Now by property {\bf (3)} in \cref{lem_update_softmax_norm} and triangle inequality, for every $i \in [m']$ and every $\k', \k'' \in \Ccal_i$ we have:
\[
\norm{\k' - \k''}_2 \le \norm{\k' - \x_i}_2 + \norm{\k'' - \x_i}_2 \le 2\delta.
\]
Therefore, using the precondition of the theorem on $\norm{\q_n}_2 \le r$ we have 
\[
\exp(\langle \k', \q_n \rangle) / \exp(\langle \k'', \q_n \rangle) \le e^{2\delta \cdot r}.
\]
Using the above inequality and the assumption in the theorem statement regarding $t = \Omega \left( \varepsilon^{-2} \cdot e^{2\delta \cdot r} \log n \right)$ combined with the properties {\bf (2)} and {\bf (5)} proved in \cref{lem_update_softmax_norm}, we can invoke Chernoff-Hoeffding inequality (see e.g., \cite{mcdiarmid1998concentration}) along with union bound to conclude that the following holds simultaneously for all $i \in [m']$ with probability at least $1 - \frac{1}{\poly(n)}$:
\[
\frac{n_i}{t} \cdot \sum_{\k' \in \Scal_i} \exp(\langle \q_n, \k' \rangle) \in (1 \pm \varepsilon/3) \cdot \sum_{\k' \in \Ccal_i} \exp(\langle \k', \q_n \rangle)
\]
Since the terms above are positive, by summing up the given inequality for all $i \in [m']$, we find that the quantity $\tau$ computed in line 27 of \cref{alg_stresm_attn_ds} satisfies the following:
\begin{equation}\label{eq:tau_approx_bound}
    \Pr \left[ \tau \in (1 \pm \varepsilon/3) \sum_{j\in[n]} \exp(\langle \k_j, \q_n \rangle) \right] \ge 0.995
\end{equation}

Next, we invoke \cref{lem_update_matrix_prod} to derive an error bound on the approximate matrix-vector product between the softmax vector and the matrix of values $\V_n$.
By leveraging well-established techniques in approximate matrix products, such as the standard result from \cite{drineas2001fast}, and using the conclusion of \cref{lem_update_matrix_prod} regarding $\Mcal$ as a list of $s = \Omega(\varepsilon^{-2} \cdot d)$ i.i.d. sample from the probability distribution $\Pr[\Mcal(j) = (\k_i, \v_i)] = \frac{\norm{\v_i}_2^2}{\sum_{l \in [n]} \norm{\v_l}_2^2}$ for $i \in [n]$ for $i \in [n]$ and $j \in [s]$, we have that vector $\z$ computed in line 26 of \cref{alg_stresm_attn_ds} satisfies the following inequality with a probability of at least $0.995$:
\begin{align}
&\norm{\z - \exp(\K_n \cdot \q_n)^\top \cdot \V_n}_2 \le \frac{\varepsilon}{3} \norm{\exp(\K_n \cdot \q_n)}_2 \normnlr{\V_n}{op} \label{eq:matrix_prod_error_bound}
\end{align}
Now by combining inequalities in \cref{eq:tau_approx_bound} and \cref{eq:matrix_prod_error_bound} using union bound and triangle inequality we find that the output of \cref{alg_stresm_attn_ds} computed in line 28 as $\z / \tau$ satisfies the following with probability at least $0.99$
\begin{align*}
    &\norm{\z / \tau - \softmax(\K_n \cdot \q_n)^\top \cdot \V_n}_2 \le \varepsilon \norm{\softmax(\K_n \cdot \q_n)}_2 \normnlr{\V_n}{op}.
\end{align*}
This completes the correctness proof.

\begin{figure*}[t]
    \centering
    \includegraphics[width=\textwidth]{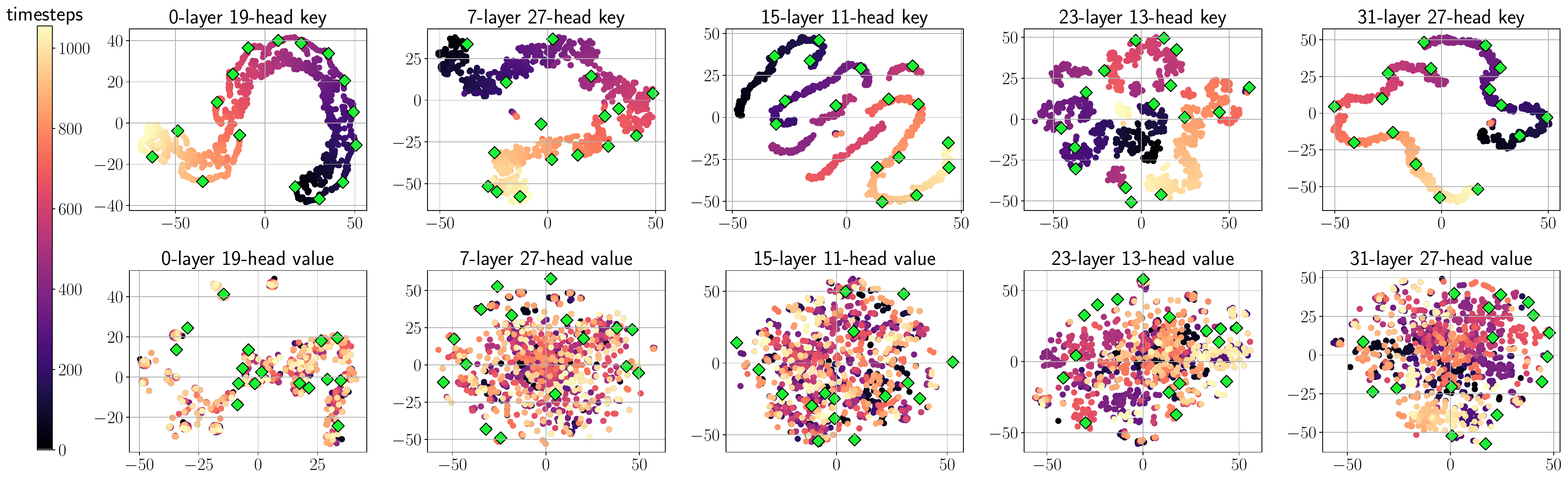}
    \vspace{-0.25in}
    \caption{A t-SNE plot of cached keys (first row) and values (second row) embeddings over $1024$ timesteps from Llama2-7B using MT Bench dataset. We pick $\ell$-layer where $\ell \in \{0,7,15,23,31\}$ and head IDs are chosen uniformly at random. Key embeddings are more clusterable than value ones. The green dots represent the centers from the greedy k-center algorithm~\cite{dyer1985simple} where k=$16$.}
    \label{fig-kv-emb}
\end{figure*}

\paragraph{Memory and Runtime.}
First, note that the memory requirement for storing the list $\Mcal$ in \cref{alg_stresm_attn_ds} is $O(s d)$ because it contains $s$ pairs of $d$-dimensional vectors.
Next, to bound the memory requirement for storing $\Dcal$ we need to bound the size of this set which we denoted by $m'$.
According to properties {\bf (1)} and {\bf (4)} in \cref{lem_update_softmax_norm}, for every $i \in [m']$ there exist $\x_i \in \{ \k_1, \k_2, \ldots \k_n \}$ such that $\norm{\x_i - \x_j}_2 > \delta$ for $i \neq j$.
Given the assumption in the theorem statement that keys are $(m, \delta)$-clusterable, by the definition of clusterability in \cref{def_clusterable} along with the pigeonhole principle, we must have $m' \le m$.
Therefore storing $\Dcal$ will require $O(m' t d) = O(m t d)$ because it is a set of $m'$ elements, and each element of this set is a list of $t$ vectors in dimension $d$.

Three major operations dominate the runtime of the $n$-th iteration. 
Firstly, executing \textsc{UpdateSoftmaxNormalizer} requires computing $m'$ distances in line 12 that takes $O(md)$ time. 
Additionally, the for loop in line 16 takes $O(td)$ time. 
Secondly, \textsc{UpdateMatrixProduct} has a runtime bounded by $O(sd)$.
Thirdly, running \textsc{QueryStreamAttn} involves $O(sd)$ operations in line 26 and $O(m'td) = O(mtd)$ operations in line 27. 
As a result, the total runtime of \cref{alg_stresm_attn_ds} in $n$-th iteration is $O(m t d + s d)$.
\end{proof}

\cref{main_theorem} demonstrates that if the keys can be clustered into some sublinear number $m = n^{1 - \Omega(1)}$ of clusters with diameters at most $\delta$, and the queries have bounded $\ell_2$ norms of at most $r$ such that the product of the cluster diameter and maximum $\ell_2$ norm of queries is bounded by $\delta r = o(\log n)$, then \cref{alg_stresm_attn_ds} operates with sublinear $O\left( \varepsilon^{-2} \cdot m d n^{o(1)} \right) = O\left( \varepsilon^{-2} \cdot d n^{1-\Omega(1)} \right)$ memory and runtime.
We summarize this in the following corollary:

\begin{corollary}\label{main_cor}
    Suppose the preconditions of \cref{main_theorem} hold. If the diameter of key token clusters $\delta$ and the maximum $\ell_2$ norm of queries $r$ satisfy $\delta r = o(\log n)$, then the total memory and runtime of \cref{main_theorem} are bounded by $O\left( \varepsilon^{-2} \cdot d m n^{o(1)} \right)$.
    Moreover, if the number of key token clusters $m$ grows as a sublinear function of $n$, i.e., as $m = n^{1 - \Omega(1)}$, then the memory and runtime are bounded by $O\left( \varepsilon^{-2} \cdot d n^{1 - \Omega(1)} \right)$.
\end{corollary}

\begin{table*}[t]
\scalebox{0.86}{
\begin{tabular}{@{}lcccccc@{}}
\toprule
 & \multicolumn{2}{c}{$n=$ 5k}               & \multicolumn{2}{c}{$n=$ 7k}               & \multicolumn{2}{c}{$n=$ 9k} \\ \midrule
Algorithm & Cache Size (GB)               & Accuracy & Cache Size (GB) & Accuracy & Cache Size (GB) & Accuracy \\ \midrule
Exact          & 2.351                     & 0.98     & 3.488                     & 1.0      & 4.613       & 0.68     \\
Sink~\cite{xiao2023efficient} & 1.511 (35\% $\downarrow$) & 0.56     & 2.012 (42\% $\downarrow$) & 0.56    &   2.262 (50\% $\downarrow$) &  0.38        \\
H2O~\cite{zhang2023h} & 1.511 (35\% $\downarrow$) & 0.66     & 2.012 (42\% $\downarrow$) & 0.58    & 2.262 (50\% $\downarrow$) & 0.38 \\
\alg~(this work)          & 1.512 (35\% $\downarrow$) & {\bf 0.86}  & 2.012 (42\% $\downarrow$) & {\bf 0.66}  & 2.262 (50\% $\downarrow$) & {\bf 0.44} \\
\bottomrule
\end{tabular}
}
\caption{Results on accuracy of line retrieval from LongEval~\cite{li2023long} dataset with context length 5k-9k. Under the sublinear budgets on cache size, the proposed algorithm based on k-center algorithm outperforms other methods over all sequence lengths.} \label{table-line-retrieval}
\end{table*}

\section{Experiments}\label{sec-exp}

In this section, we report the empirical results of the proposed algorithm with memory footprint reduction and performance on question-answering benchmark datasets. For all experiments, we use a single NVIDIA RTX6000 GPU with 48 GB VRAM.

\subsection{Ablation Study on Clusterability}\label{exp-key-embedding}
We first demonstrate that cached embeddings over long token generations are indeed well clusterable. To this end, we collect key and value embeddings from Llama-2-7B~\cite{touvron2023llama} with MT Bench dataset~\cite{zheng2023judging} while the model generates a sequence of $1024$ tokens. We then visualize the embeddings using t-SNE~\cite{van2008visualizing} across various layers and heads, identifying cluster center points through the greedy k-center algorithm~\cite{dyer1985simple}.

As illustrated in \cref{fig-kv-emb}, our observations reveal that key embeddings (first row) exhibit a higher degree of clusterability compared to value embeddings across all randomly selected layers and heads. Furthermore, we note that the cluster centers (indicated by green dots) corresponding to the key embeddings are evenly distributed across the entire embedding space. In particular, the key embeddings demonstrate significant dispersion across different time steps, and their cluster centers are distributed over the entire embedding space.

This behavior comes from the use of Rotary Position Embedding (RoPE)~\cite{su2024roformer} in Llama-2-type models which introduces rotational transformations to both query and key embeddings based on their relative positions. Hence, the key embeddings appear to be well-separated in their projected space, while the values show an unstructured and random distribution within their space. These findings serve as a motivation for the development of an efficient key-value (KV) compression algorithm that leverages the clustering properties of key embeddings.

\subsection{End-to-end Text Generation}

We next evaluate our proposed algorithm on long-context line retrieval task in LongEval~\cite{li2023long}\footnote{\url{https://github.com/DachengLi1/LongChat/blob/longeval}} benchmark. The task involves long-context line retrieval from extensive documents, each comprising multiple lines, complete with line numbers and topics. The objective is to precisely retrieve a specified number of lines corresponding to a target topic. We vary the number of lines, representing the number of targets, to 200, 300, and 400 and they correspond to sequence lengths of $n=$5,000, 7,000, and 9,000, respectively. Each dataset contains 50 distinct questions, and we systematically extract the number from the generated answers and compute accuracies. The answers are generated employing the longchat-7B model\footnote{\url{https://huggingface.co/lmsys/longchat-7b-v1.5-32k}}, which is a fine-tuned version of the Llama-2-7B model with long-range context length.

We compare our method to two KV cache compression algorithms; H2O~\cite{zhang2023h}, which retains cached tokens with high cumulative attention scores, and Attention Sink~\cite{xiao2023efficient}, a method that deterministically selects some initial and recent tokens. Specifically, both of these prior works have highlighted the significance of recent token embeddings in generating meaningful responses. To leverage this insight, we integrate it with our clustering approach. More precisely, our strategy consistently retains the most recent $r$ token embeddings, in addition to $k$ centers selected from the remaining tokens. In a streaming context, this strategy is often referred to as a {\it sliding window}. We apply the greedy k-center clustering algorithm once to compress the entire KV caches. To make comparisons fair, we set cache memory budgets of all algorithms identical (i.e., $r+k$), which scales sublinearly with the context length denoted as $n$.

The results are reported in \cref{table-line-retrieval}. We observe that our clustering-based method consistently outperforms other algorithms across all sequence lengths. For instance, we achieve an accuracy of 44\% while utilizing only half of the cached KV embeddgins with a length of 9k tokens, whereas both H2O and Sink can achieve accuracies 10\% lower. This finding suggests that maintaining the embedding information holds greater significance in sustaining the performance of LLMs compared to attention scores and positional information.

\section{Conclusion}
In this work, we develop \alg, an efficient KV cache compression algorithm via stream clustering.
Our motivation is that cached keys are well clusterable in their embedding space and we apply a greedy-type clustering algorithm to find the most representative embeddings. 
Under assumptions on bounded query norm and clusterability, we analyze that our algorithm can guarantee a spectral error bound with sublinear time and memory. 
We further integrate keeping recent tokens to the proposed clustering approach. For zero-shot line retrieval tasks, our algorithm outperforms other KV cache compression algorithms with the same memory budget.

\bibliographystyle{icml2024}
\bibliography{references}

\end{document}